\documentclass{amsart}

\usepackage{amsmath, amssymb, amsthm, graphicx, color}
\usepackage{hyperref}

\theoremstyle{plain} \numberwithin{equation}{section}
\newtheorem{theorem}{Theorem}[section]
\newtheorem{proposition}[theorem]{Proposition}
\newtheorem{corollary}[theorem]{Corollary}
\newtheorem{lemma}[theorem]{Lemma}

\newtheorem{definition}[theorem]{Definition}
\newtheorem*{theorem*}{Theorem}
\newtheorem*{proposition*}{Proposition}

\theoremstyle{remark}

\newtheorem*{example*}{Example}

\newcommand{\R}{\mathbb R}

\newcommand{\calP}{\mathcal{P}}

\begin{document}

\title{The critical locus of overparameterized neural networks}
\author{Y. Cooper}
\email{\href{mailto:yaim@math.ias.edu}{yaim@math.ias.edu}}

\maketitle

\begin{abstract}
Many aspects of the geometry of loss functions in deep learning remain mysterious.  In this paper, we work toward a better understanding of the geometry of the loss function $L$ of overparameterized feedforward neural networks.  In this setting, we identify several components of the critical locus of $L$ and study their geometric properties.  For networks of depth $\ell \geq 4$, we identify a locus of critical points we call the star locus $S$.  Within $S$ we identify a positive-dimensional sublocus $C$ with the property that for $p \in C$, $p$ is a degenerate critical point, and no existing theoretical result guarantees that gradient descent will not converge to $p$.  For very wide networks, we build on the work of \cite{quynhconnectedsublevel} and \cite{ruoyumeasure0} and show that all critical points of $L$ are degenerate, and give lower bounds on the number of zero eigenvalues of the Hessian at each critical point.  For networks that are both deep and very wide, we compare the growth rates of the zero eigenspaces of the Hessian at all the different families of critical points that we identify.  The results in this paper provide a starting point to a more quantitative understanding of the properties of various components of the critical locus of $L$.  
\end{abstract}

\section{Introduction}
The recent and remarkable success of neural networks is not yet well understood from a theoretical perspective.  A fruitful area of study has been the ``expressivity'' of deep neural networks.  That is, given a problem or data set, how large does your neural network need to be so that there exists a function in the corresponding parameter space $\calP$ that perfectly fits the training data set?

However, even if there exist parameters in $\calP$ that encode functions with zero training loss, it is far from clear if and when gradient based methods might find such parameters.  In most cases, the loss function $L$ is believed to be nonconvex, and under various assumptions, ``bad'' critical points have been proven to exist.  In other words, in most cases, we expect there to be many critical points that gradient descent could get stuck at, and still it appears that empirically gradient descent often finds global minima.

In this work, we are motivated by this remarkable fact.  Figure \ref{minima} illustrates the long term goal --- not only would we like to identify all the critical points of $L$, but we'd like to understand the geometry of the set $Crit$ of all critical points of $L$.  For example, it would be valuable to understand how many components $Crit$ has, what the dimension of each component is, and what the local geometry of $L$ is near each component of $Crit$.  In this paper, we make some of the first steps toward this goal, establishing fundamental facts about several components of the locus of critical points of $L$.

\begin{figure}\label{minima}
\begin{center}
\includegraphics[width=4.2in]{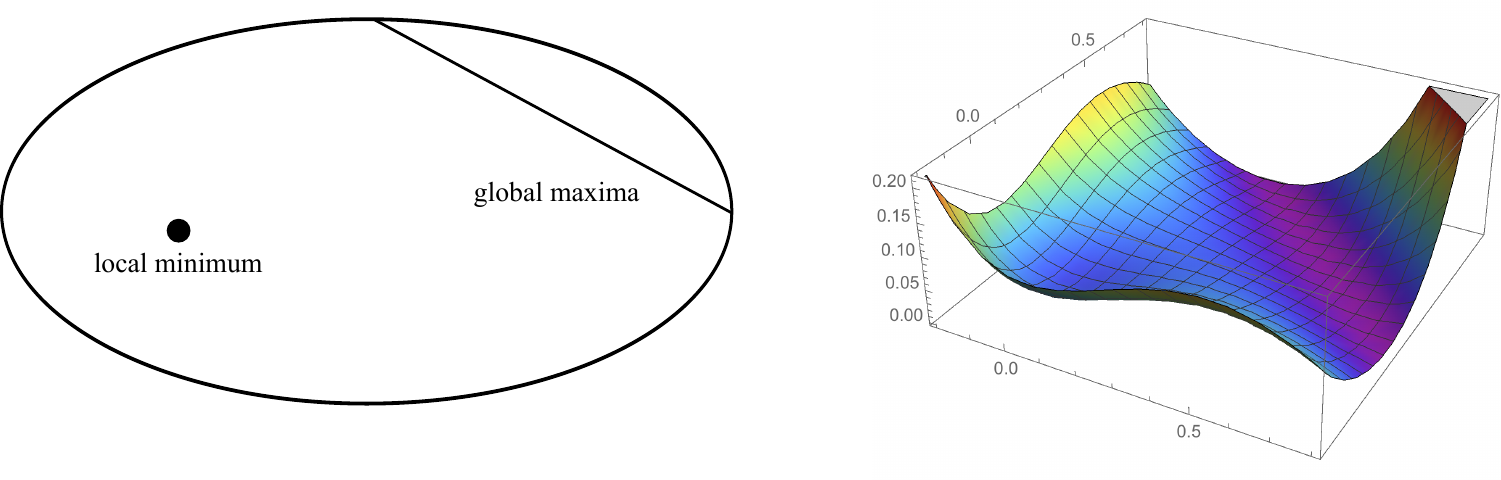}
\end{center}
\caption{The locus of critical points of this function has two components.  There's one isolated local minimum, and then there's a one dimensional smooth manifold of global minima.}
\end{figure}

\subsection{Previous results}
Not much is currently known about the critical locus of the loss function under assumptions that hold for real-world neural networks.  There is a substantial literature on the geometry of the loss function under strong assumptions that rule out many or all real-world networks, for example, that the activation function is the identity, or that the network is unrealistically wide compared to the number of training samples.

We begin by recalling some results about the critical points of $L$ that hold under assumptions mild enough to include real-world architectures.  The author studied the locus of global minima in \cite{cooper}, where under mild assumptions we showed that for feedforward networks, the locus of global minima of $L$ forms a smooth, possibly disconnected manifold $M$ and computed its dimension.  

Whether $M$ is connected has been studied both theoretically and empirically.  In \cite{quynhconnectedsublevel}, Ngyuen et al found that under stronger assumptions, it is possible to prove that the locus of global minima is in fact connected.  Working in a different setting, in\cite{venturibruna} Venturi et. al. also prove the connectedness of the locus of global minima under strong assumptions.  

In related work \cite{sanjeevepsilonconnected}, Arora et al showed that under a different but again strong set of assumptions, any two dropout-stable minima are $\epsilon$-connected.  It has also been empirically observed by several groups including \cite{LossSurfaces} and \cite{NoBarriers}, that if one trains a neural net twice, two different global minima are found, but that it is often possible to construct a path between them along which the loss does not increase much along the path.

The geometry of the locus of local non-global minima is much less well understood than that of global minima.  Kawaguchi showed that for deep linear networks, all local minima of the loss function are global \cite{NoPoor}.  However, no real-world neural networks use the identity function as the activation function $\sigma$, and once $\sigma$ is allowed to be nonlinear, we know very little about the local non-global minima of $L$.  Several groups, including \cite{srabadlocalmin}, \cite{spuriouscommon}, and \cite{ruoyumin}, have proven, under various sets of assumptions that are weak enough to include some real-world networks, that deep nonlinear neural networks always have spurious local minima.  

The geometry of $L$ near critical points that are neither minima nor maxima is perhaps the least understood aspect of the landscape of the loss function for deep nonlinear networks.  Very little is currently known about these other critical points of $L$, beyond the expectation that $L$ has nondegenerate saddle points and the fact that the origin is often a degenerate critical point.  

We pause here to recall what kinds of critical points a smooth function $L$ can have, and which of those we might expect gradient-based methods to find.  Critical points of $L$ fall into two broad categories --- nondegenerate critical points, where all the eigenvalues of the Hessian are nonzero, and degenerate critical points, where one or more eigenvalues of the Hessian vanishes.  While nondegenerate critical points are easily classified by the number of positive and negative eigenvalues of the Hessian, there is a rich zoo of degenerate critical points.  

As to which critical points gradient-based methods might converge to, one does not expect gradient descent from a random initialization to end at a local maximum, and one does expect sometimes to end at a local minimum.  The most subtle are saddle points, especially degenerate ones.  

In \cite{jordansaddle}, Jordan et. al. showed that gradient-based methods can efficiently escape any saddle point $p$ where at least one eigenvalue of the Hessian is negative.  However, beyond that, there is little understanding of which kinds of degenerate critical points gradient based methods can efficiently escape and which they cannot.  In this paper, we show that for any deep nonlinear network, overparameterized or not, $L$ has a positive dimensional locus of degenerate critical points which do not satisfy the assumptions of \cite{jordansaddle}, and which to our knowledge no existing results guarantee gradient descent won't get stuck at.

Given the fact then that $L$ contains degenerate critical points with geometries that haven't been studied from the point of view of gradient descent, a valuable line of future research would be to extend the work of \cite{jordansaddle} to larger classes of degenerate critical points.  One would like to establish which kinds of degenerate critical points can be problematic for gradient-based methods and which can kinds we can guarantee will not be.  It would also be valuable to understand which of the problematic kinds of degenerate critical points in fact appear as critical points of the loss function $L$.  

\subsection{Our contribution}
In this work, we establish some basic geometric properties of the locus of critical points of the loss function of deep nonlinear neural networks.  We treat all three classes of critical points that gradient descent could with positive probability converge to --- global minima, local non-global minima, and degenerate saddle points.  

For networks of depth $\ell \geq 3$, with smooth activation function satisfying $\sigma(0)=0$, whether overparameterized or not, we show the existence of a set $S$ we call the star locus, which is a positive dimensional locus of critical points (there may be others.)  

\medskip

\noindent
{\bf Theorem \ref{starlocus}}
Every point in $S$ is a critical point of $L$.  

Within the star locus we identify a set we call the core locus $C$, which is a positive-dimensional locus of degenerate critical points.  

\medskip

\noindent
{\bf Theorem \ref{corelocus}}
Every point in $C$ is a degenerate critical point.  Furthermore, for any point $p$ in $C$, the Hessian of $L$ at $p$ has one positive eigenvalue and the remaining eigenvalues all vanish.  

\medskip

\noindent
We go on to compute the dimension of both loci $S$ and $C$, and establish some basic properties of each.

In practice, modern neural networks are essentially always overparameterized, in the sense that the number of parameters $d$ is larger than the number of data points $n$ that the network is being trained on.  In this setting, in previous work \cite{cooper} we showed that the locus $M$ of global minima is generically a smooth manifold of dimension $d-bn$.  

Little more is known about neural networks which are overparameterized but to an extent comparable to the extent that networks used in practice are overparameterized.  However, there has been substantial recent progress in the theoretical understanding of networks much wider than those used in practice.  We call networks where each layer has width $m > n$ very wide, and those where the width $m$ of each layer grows as a polynomial in $n$ extremely wide.  

Our best theoretical understanding is for extremely wide networks, where work of \cite{dulee}, \cite{zhu}, \cite{ntk}, and many others proves that with high probability, gradient flow from random initializations converge to global minima.  Our theoretical understanding of very wide networks is more limited, but has also seen substantial progress in the past few years.  

In \cite{fullrank}, \cite{quynhconnectedsublevel}, \cite{ruoyumeasure0}, and \cite{venturibruna}, several groups showed in a number of different settings that the sublevel sets of $L$ are connected.  In particular, this implies that in this setting all critical points of $L$ are degenerate.  In this work, we build on their work and give lower bounds on the number of globally flat directions of the function $L$ at any critical point $p$.  
\medskip

\noindent
{\bf Theorem \ref{dimlinuniform}.}
Consider a very wide feedforward neural network with smooth activation $\sigma$ and $L2$ loss $L$.  For any critical point $p$ of $L$, the level set of $p$ contains a linear subspace of dimension $(m_{\ell-1}+1 - n) m_\ell$.  

\medskip

A corollary of this Theorem stated in the perhaps more familiar language of zero eigenvalues of the Hessian of $L$ is the following.

\medskip

\noindent
{\bf Corollary \ref{critzeroeigenuniform}.}
Consider a very wide feedforward neural network with smooth activation $\sigma$ and $L2$ loss $L$.  For any critical point $p$ of $L$, $Hess(L)$ at $p$ has at least $(m_{\ell-1}+1 - n) m_\ell$ zero eigenvalues.
\medskip

We also show that local minima of $L$ cannot be isolated in the set of critical points, meaning that if $p$ is a local minimum, any $\epsilon$ neighborhood of $p$ must contain another local minimum.

\medskip

\noindent
{\bf Proposition \ref{noisolatedmin}.}
For a fully connected feedforward neural network with width $m$ larger than the number of training samples $n$, continuous activation function $\sigma$ and L2 loss function $L$, $L$ has no isolated local minima.

\medskip

We show an analogous result for local maxima, but are unable to prove an analogous result for any other critical points, and it remains an open question whether with very wide networks $L$ can have isolated saddle points.  

Finally, we compare the growth rates of the dimension of the loci of all the types of critical points we discuss in this paper, as well as the growth rates of the zero eigenspaces of the Hessian of $L$ at each type.

\medskip

\noindent
{\bf Proposition \ref{dimloci}.}
Consider a family of feedforward neural networks with hidden layers of increasing width training on a fixed data set.  That is, let $a, b,$ and $n$ be fixed while $m$ increases.  Then the dimensions of the locus of global minima, star locus, and core locus, are:
\begin{align*}
\dim(M) &= (\ell-2)m^2 + (a+b+\ell-1)m + b(1-n),
\\
\dim(S) &=  (\ell-3)m^2 + (a+1)m,
\\
\dim(C) &=  (\ell-4)m^2 + (a+1)m,
\end{align*}
while the dimension of the locus of all critical points is unknown.

\medskip

Meanwhile,

\medskip

\noindent
{\bf Proposition \ref{dimzeroeigen}.}
Consider a family of feedforward neural networks with hidden layers of increasing width training on a fixed data set.  That is, let $a, b,$ and $n$ be fixed while $m$ increases.  Then the number of zero eigenvalues of the Hessian of $L$ at any critical point in the core locus, locus of global minima, or any critical point are:
\begin{align*}
\begin{cases}
d-1 & \text{for } p \in C,
\\
d-bn & \text{for } p \in M,
\\
\text{at least } (m+1- n) b & \text{for all other critical points } p.
\end{cases} 
\end{align*}

\medskip

\subsection{Outline of paper}

We begin, in Section \ref{setting}, by establishing the setting in which we will be working.  We then study the geometry of $L$ for neural networks of various sizes, starting with underparameterized networks and then working with increasingly overparameterized ones.

In Section \ref{deepnetwork} we consider all neural networks of depth $\ell \geq 3$, whether overparameterized or not.  In this setting, we identify two families of critical points, the star locus $S$, and the core locus $C$.  When$\ell \geq 3$ the star locus is positive dimensional, and once $\ell \geq 4$, the core locus is as well.  The core locus consists of critical points with nearly vanishing Hessian --- all but one eigenvalue is zero, and the one nonzero eigenvalue is positive.  Hence every point in this positive dimensional locus of critical points fails the assumptions of \cite{jordansaddle}, and we cannot from that work conclude that gradient descent will not converge to them.

Next, in Section \ref{overparam} we consider all neural networks with more parameters $d$ than data points $n$, and recall from earlier work that the locus of global minima forms a smooth possibly disconnected manifold of dimension $d-bn$.  

Finally in Section \ref{verywide}, we consider neural networks which are not only overparameterized, but where every layer has width greater than $n$.  This setting has been of substantial recent theoretical interest, and here we show that there are no isolated local minima or maxima.  We go on to give a lower bound on the number of zero eigenvalues of the Hessian at any local minimum or maximum.  For saddle points, we give a lower bound on the dimension of certain linear subspaces that contain the saddle point and are contained in the level set containing the saddle point.

We conclude in Section \ref{discussion} with a comparison, for networks in which the kinds of critical points discussed in this paper all appear simultaneously, of the relative dimensions of each type, and how they grow with the number of parameters in the network.  

\subsection{Acknowledgements}
We thank Misha Belkin, Nate Bottman, Rong Ge, Felix Janda, Chi Jin, Holden Lee, and Ruoyu Sun for helpful discussions.  We thank Quynh Ngyuen for valuable feedback on an early draft of this work.

\section{Setting and notation}\label{setting}

In this paper, we will consider the following setting.  Consider a fully connected feedforward neural network with $L2$ loss, and a monotonically increasing activation function $\sigma$.  Further, assume that $\sigma(0)=0$ and $\sigma$ is smooth, meaning it is infinitely differentiable. 

Suppose the neural network is training on a data set $D$ consisting of input output pairs $(x_k, y_k)$, $x_k \in \R^a$, $y_k \in \R^b$, and suppose there are $n$ data points in the training set.  

We assume that the neural network has $\ell-1$ hidden layers and each layer has width at least $m$.  Let the width of the $i^{th}$ layer be denoted $m_i$, so we assume each $m_i \geq m$ for $1 \leq i \leq \ell-1$.  Let $m_0 = a$ and $m_\ell = b$.  Let the number of parameters of the network be denoted by 
$$d = \sum_{i = 1}^{\ell} (m_{i-1} + 1)m_i.$$

\begin{figure}\label{netdiagram}
\begin{center}
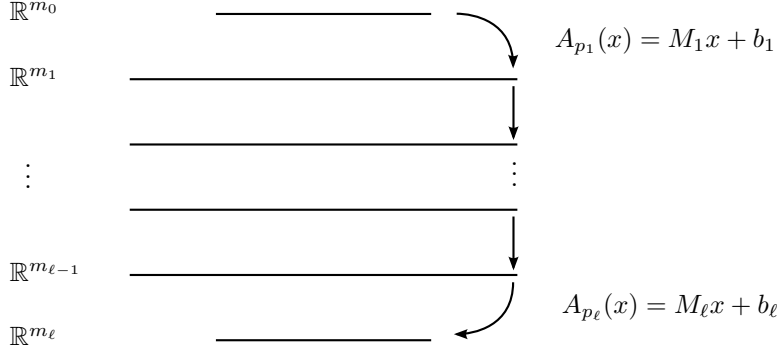
\end{center}
\caption{A schematic diagram of the feedforward networks we consider.}
\end{figure}

A fully connected feedforward network of this form parameterizes functions of the form
\begin{equation} \label{basicformula}
M_\ell(\sigma(M_{\ell-1}( ... ( \sigma (M_1 x + b_1) ... ) + b_\ell
\end{equation}
where each $M_i x + b_i$ is an affine linear transformation from $\R^{m_{i-1}}$ to $\R^{m_i}$.  

We sometimes denote the affine linear transformation $x \rightarrow M_i x + b_i$ by $A_i$.  In this notation (\ref{basicformula}) can be expressed as 
\begin{equation*}
A_\ell \circ \sigma \circ A_{\ell-1} \circ ... \circ \sigma \circ A_1(x).
\end{equation*}

In Figure \ref{netdiagram}, we diagram the feedforward neural network we have just described.  

The space of all weights and biases for all the layers is isomorphic to $\R^d$.  It will be helpful to distinguish the weights and biases for each layer, so let 
$$\R^{d_i} = \{ (M_{i-1}, b_{i-1} )\}$$ 
denote the space of weights and biases that determine the affine map from the $(i-1)^{st}$ to the $i^{th}$ layer.  Then $d_i = m_i (m_{i-1}+1)$ and
\begin{align*}
\R^d = \R^{d_1} \times ... \times \R^{d_\ell}.
\end{align*}
Given a parameter vector $p \in \R^d$, we can decompose it as 
\begin{align*}
p = (p_1, ..., p_\ell).
\end{align*}

The loss function $L$ is a function from $\R^d$ to $\R$ defined as
$$L(p) = \sum_{i = 1}^{n} (f_{p}(x_i) - y_i)^2,$$
where $p$ is a parameter vector and $f_p$ is the function computed by the neural network with that choice of parameters.  

Our interest is in understanding the geometry of the loss function $L$ of overparameterized neural networks, as the majority of modern neural networks deployed in practice today are overparameterized.  

\subsection{Overparameterized regimes}
We now outline three regimes of overparameterization.  The more overparameterized the neural network, the better our current understanding of the geometry of $L$.  

\subsubsection{Mildly overparameterized regime}
The most realistic and hence most interesting case is when the neural network is overparameterized in the simplest sense, that the number of parameters $d$ of the network is greater than the number of training points $n$, but not excessively so, i.e. that $d = O(n)$.  In this case, the width $m$ of the network is order $\sqrt{n}$.  We call this the mildly overparameterized regime.

\subsubsection{Very wide regime}
The second range we consider is the case that not only is the network overparameterized in the sense of having more parameters than data points, but that furthermore the width $m$ of the network is larger than the number of data points $n$.  We don't consider arbitrary width here, but rather the case that $m>n$ but that $m$ is still linear in $n$.  In this case, the number of parameters $d$ is quadratic in the number of data points $n$.

We call this the very wide regime.  It is wider than the networks used in practice, but not in an extreme way, and there is better theoretical understanding of the geometry of $L$ in this range than in the mildly overparameterized regime.

\subsubsection{Extremely wide regime}
The final range we describe is the most heavily overparameterized.  Here, we assume that the width $m$ of the neural network is polynomial in the number of data points $n$.  In this setting, many authors even take the infinite width limit and let $m$ go to $\infty$.

We call this the extremely wide regime, and it is much more overparameterized than any neural network used in practice.  This regime is unrealistic for real world neural networks, but on the other hand, this is the setting in which we have the best understanding of the geometry of $L$ as well as the dynamics of gradient descent on $L$.  So this case is also often of interest.

\section{Deep neural networks}\label{deepnetwork}
In this section, we establishing some basic facts about the loss function $L$ that hold for all deep neural networks, regardless of whether they are overparameterized or not.  

\subsection{Star locus}
For any fully connected feedforward network with $L2$ loss and smooth activation function satisfying $\sigma(0) = 0$, the space of all parameters $\R^d$ contains a positive dimensional locus of critical points we call the star locus, which contains a sublocus of degenerate critical points we call the core.  In this section we will identify these loci and prove some properties about them.

We begin by recalling some definitions.  Given a twice differentiable function $f: \R^d \rightarrow \R$, $p \in \R^d$ is a critical point if
$$
\left(\frac{\partial}{\partial z_i} f \right)(p) = 0
$$
for all $i$.  

The Hessian of $f$ is defined as the matrix
$$
Hess(f) =
\begin{pmatrix}
\frac{\partial^2}{\partial z_1^2} f & \dots & \frac{\partial }{\partial z_1}\frac{\partial}{\partial z_d} \\
\vdots & & \vdots \\
\frac{\partial }{\partial z_d}\frac{\partial }{\partial z_1} f & \dots & \frac{\partial^2}{\partial z_d^2} f \\
\end{pmatrix}
$$

A point $p$ is a degenerate critical point if all the derivatives of $f$ vanish at $p$ and in addition the Hessian of $f$ does not have full rank at $p$.  In other words,
$$det(Hess(f)) (p) = 0.
$$

Finally, we say that $p$ is a critical point of order $k$ if all of the derivatives of $f$ up to order $k$ vanish at $p$.  For example, if $p$ is a critical point and the Hessian of $f$ at $p$ is identically zero, then $p$ is a critical point of order at least 2.  
 
\subsubsection{Description}
Consider a feedforward neural network as above,
$$A_\ell \circ \sigma \circ ... \circ \sigma \circ A_1 = M_\ell(\sigma(M_{\ell-1}(.... ( \sigma (M_1 x + b_1) ... ) + b_\ell.$$

\noindent
For any $1 \leq k \leq \ell-1$ we define the locus $S_k$ by the following formula:
\begin{align*}
S_k
=
\left\{
p = (M_1,b_1,\ldots,M_\ell,b_\ell)
\:\left|\:
M_\ell = M_k = 0,
b_\ell = \sum_{\alpha=1}^n y_\alpha,
b_k=\cdots=b_{\ell-1}=0
\right.\right\}.
\end{align*}

\noindent
We define the star locus $S$ as the union of all these linear subspaces $S_k$, for $1 \leq k \leq \ell-1$.  

Note that the dimension of the star locus is 
\begin{align*}
\dim(S) = \text{Max}_{1 \leq i \leq \ell-1} d - \left( (m_{\ell-1}+1)m_\ell + m_{i-1}m_i + \sum_{j=i}^{\ell-1} m_j \right)
\end{align*}

For each choice of $m$ integers $1 \leq k_1 \leq \cdots \leq k_m \leq \ell-1$, we define the locus $C_{k_1, \ldots, k_m}$ by the following formula:
\begin{gather}
C_{k_1,\ldots,k_m}
=
\left\{\left.
p = (M_1,b_1,\ldots,M_\ell,b_\ell)
\:\right|\:
\eqref{eq:C_conds}
\right\},
\nonumber
\\
M_\ell = 0,
b_\ell = \sum_{\alpha=1}^n y_\alpha,
M_{k_1}=\cdots=M_{k_m}=0,
b_{k_1}=\cdots=b_{\ell-1}=0.
\label{eq:C_conds}
\end{gather}

\noindent
For each $m$, we define the $m$-core $C_m$ as the union over all $m-$tuples $k_1, ..., k_m$ of $C_{m_1, ..., m_k}$.  Note that the star locus $S$ is equal to the $1$-core $C_1$. 

Note that 

$$(0,...,0) = C_\ell \subset ... C_k \subset C_{k-1} ... C_2 \subset C_1 = S$$
and that each $C_k$ is the union of linear subspaces.  We call $C_2$ the core locus, and denote it by $C$.  This locus will be of particular interest to us, and it contains all the higher cores $C_3, C_4$, and so on.

\begin{figure}\label{starcore}
\begin{center}
\includegraphics[width=2.2in]{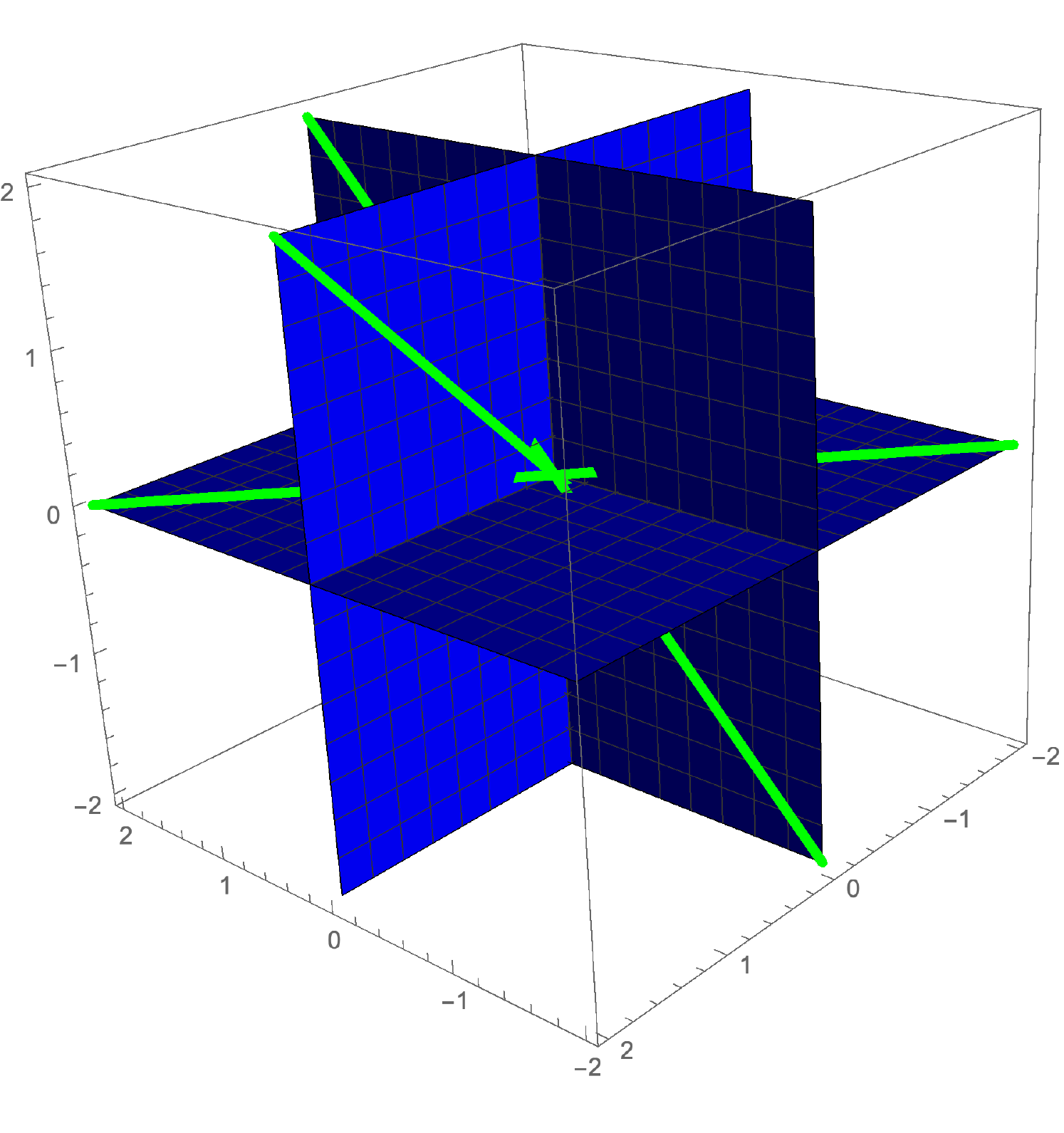}
\end{center}
\caption{The star locus is a union of linear spaces, and looks something like the blue set in this image.  The core locus is a subset of the star locus, is also a union of linear spaces, and looks something like the green set in this image.}
\end{figure}

Note that the dimension of the core locus is 
\begin{align*}
dim(C) = \text{Max}_{1 \leq i<j \leq \ell-1} d - \left( (m_{\ell-1}+1)m_\ell + m_{i-1}m_i + m_{j-1}m_j + \sum_{k=i}^{\ell-1} m_k \right)
\end{align*}

\subsubsection{Vanishing of derivatives}

Now, we will prove that for deep networks, every point in the star locus $S$ is a critical point and every point in the core locus $C$ is a degenerate critical point.

\begin{theorem} \label{starlocus}
Given a fully connected feedforward network of depth $\ell \geq 3$, with smooth activation function satisfying $\sigma(0) = 0$ and L2 loss function $L$, every point in the star locus $S$ is a critical point.  
\end{theorem}

\begin{theorem} \label{corelocus}
Given a fully connected feedforward network of depth $\ell \geq 4$, with smooth activation function satisfying $\sigma(0) = 0$ and L2 loss function $L$, every point in the core locus $C$ is a degenerate critical point.  Furthermore, for any point $p$ in $C$, the Hessian of $L$ at $p$ has one positive eigenvalue and the remaining eigenvalues all vanish.  \end{theorem}

The loss function is
\begin{align*}
L(p) = \sum_{\alpha=1}^n (f_p(x_\alpha) - y_\alpha)^2.
\end{align*}

This means that differentiation with respect to any parameter $*$ gives
\begin{align}\label{onederivative}
\frac{\partial}{\partial *} L(p) = \sum_{\alpha=1}^n 2 (f_p(x_\alpha) - y_\alpha) \frac{\partial}{\partial *} f_p(x_\alpha).
\end{align}

In the deep nonlinear case, the function computed by the neural network is of the following form:

\begin{equation} \label{f}
\begin{aligned}
f_p(z) = 
\begin{pmatrix}
w_{11}^\ell &...& w_{1 m_{\ell-1}}^\ell &b_1^\ell \\
&\vdots & & \\
w_{m_\ell 1}^\ell &...& w_{m_\ell m_{\ell-1}}^\ell &b_{m_\ell}^\ell \\
\end{pmatrix}
\circ \sigma \circ
\begin{pmatrix}
w_{11}^{\ell-1} &...& w_{1 m_0}^1 &b_1^{\ell-1}\\
&\vdots & & \\
w_{m_1 1}^{\ell-1} &...& w_{m_1 m_0}^1 &b_{m_1}^{\ell-1} \\
0 & ... & 0 & 1 \\
\end{pmatrix}&
\circ \sigma \circ \\
...
\circ \sigma \circ
\begin{pmatrix}
w_{11}^1 &...& w_{1 m_0}^1 &b_1^1\\
&\vdots & & \\
w_{m_1 1}^1 &...& w_{m_1 m_0}^1 &b_{m_1}^1 \\
0 & ... & 0 & 1 \\
\end{pmatrix}
\begin{pmatrix}
z_1 \\
\vdots \\
z_{m_0}\\
1 \\
\end{pmatrix}&
\end{aligned}
\end{equation}

Thus an entry of $f_p(z)$ will have the form

\begin{align}\label{fentry}
\sum w..^\ell \sigma \left( \sum w..^{\ell-1} ... \sigma \left( \sum w..^2 \sigma \left( \sum w..^1 z. + b.^1 \right) + b.^2 \right) + ... + b.^{\ell-1} \right) + b.^\ell.
\end{align}
For us, it will be sufficient to keep track of the shapes of the terms that appear, without keeping track of the indices.  

\begin{lemma} \label{notell}
Any derivative of any order of $f_p(z)$ with respect to the parameter variables that does not involve $\frac{\partial}{\partial w_{..}^\ell}$ derivatives vanishes.  The only exception is the first derivative $\frac{\partial}{\partial b_{.}^\ell} f_p(z)$.
\end{lemma}

\begin{proof}
Let $D$ be a derivative satisfying the assumptions.  Using expression \ref{fentry},
\begin{align*}
D f_p(z) &= D \left( \sum w..^\ell \sigma \left( \sum w..^{\ell-1} ... \sigma \left( \sum w..^1 z. + b.^1 \right) + ... + b.^{\ell-1} \right) + b.^\ell \right)  \\
& = D \left(\sum w..^\ell \sigma \left( \sum w..^{\ell-1} ... \sigma \left( \sum w..^1 z. + b.^1 \right) + ... + b.^{\ell-1} \right) \right) + D( b.^\ell)
\end{align*}
Since $D$ does not contain any $w..^\ell$ derivatives, the first term vanishes.  And except for the derivative $\frac{\partial}{\partial b.^\ell}$, any derivative of $b.^\ell$ vanishes, so the second term vanishes as well.
\end{proof}

\begin{proof}[Proof of Theorem \ref{starlocus}]
To show that every point in $S = \bigcup S_k$ is a critical point, we take any $k$ and any $p \in S_k$ and show that $\frac{\partial}{\partial *} L(p) = 0$ for any parameter $*$.  

For most parameters, we will show that $\frac{\partial}{\partial *} f_p(x_\alpha) = 0$ for all $x_\alpha$.

For all $i \neq \ell$, the derivatives 
$$
\frac{\partial}{\partial w..^i} f_p(x_\alpha) 
$$
and
$$
\frac{\partial}{\partial b.^i} f_p(x_\alpha) 
$$
vanish, by Lemma \ref{notell}.  

Next, we check the derivative
$$\frac{\partial}{\partial w..^\ell} f_p(x_\alpha).$$

The coordinates of this derivative are sum of terms of the form
\begin{align*}
\frac{\partial}{\partial w..^\ell} \left( \sum w..^\ell \sigma \left( \sum w..^{\ell-1} ... \sigma \left( \sum w..^2 \sigma \left( \sum w..^1 z. + b.^1 \right) + b.^2 \right) + ... + b.^{\ell-1} \right) + b.^\ell
\right) \\
\end{align*}
which are either 0 or of the form
\begin{align*}
\sigma \left( \sum w..^{\ell-1} ... \sigma \left( \sum w..^2 \sigma \left( \sum w..^1 z. + b.^1 \right) + b.^2 \right) + ... + b.^{\ell-1} \right) \\
\end{align*}

Every expression of this form vanishes on $S$, because every $w..^k = 0$, and for $i > k$, every $b.^i$ is also zero.  So this becomes
\begin{align*}
\sigma \left( \sum w..^{\ell-1} \sigma \left( ... \sigma \left( \sum w..^{k+1} \sigma \left( 0 \right) + 0 \right) ... + 0 \right) \right) \\
\end{align*}
which vanishes because $\sigma(0) = 0$.  

Hence for all the derivatives considered thus far, $\frac{\partial}{\partial *} f_p(x_\alpha) = 0$ for all $x_\alpha$.  Consider the expression (\ref{onederivative}) 
$$\frac{\partial}{\partial *} L(p) = \sum_{\alpha=1}^n 2 (f_p(x_\alpha) - y_\alpha) \frac{\partial}{\partial *} f_p(x_\alpha).$$
In this derivative, the second factor vanishes for every $\alpha$, hence $\frac{\partial}{\partial *} L(p)$ vanishes for all the parameters discussed.

We analyze the final group of derivatives, those with respect to $b.^\ell$, differently.  For $b.^\ell$, there is only one coordinate in which the derivative could be nonzero.  In this component, the second factor of (\ref{onederivative}) is 1.  Thus we are left with 
$$\frac{\partial L}{\partial b.^\ell} = \sum_{\alpha=1}^n 2 (f_p(x_\alpha) - y_\alpha)$$
This vanishes by our choice of $b^\ell$.  This suffices to show that 
$$\frac{\partial}{\partial b.^\ell} L(p) = 0.$$
\end{proof}

\begin{proof}[Proof of Theorem \ref{corelocus}]

Fix any point $p$ in $C$.  By Theorem \ref{starlocus}, $p$ is a critical point.  We will now show that every second derivative of $L$ with respect to the parameters $\{w_i, b_j\}$ vanishes, except for
$$
\frac{\partial^2}{\partial {b.^\ell}^2}= 2n.
$$
This suffices to prove the claim that for every point $p$ in $C$, the Hessian of $L$ at $p$ has one positive eigenvalue and the remaining eigenvalues all vanish.

Differentiation with respect to any two parameters $\gamma_1, \gamma_2$ gives
\begin{align}\label{twoderivatives}
\frac{\partial}{\partial \gamma_1} \frac{\partial}{\partial \gamma_2} L(p) = 2 \sum_{\alpha=1}^n \frac{\partial}{\partial \gamma_1} f_p(x_\alpha) \frac{\partial}{\partial \gamma_2} f_p(x_\alpha) +(f_p(x_\alpha) - y_\alpha) \frac{\partial}{\partial \gamma_1} \frac{\partial}{\partial \gamma_2} f_p(x_\alpha).
\end{align}

By Lemma \ref{notell}, at any point $p \in C$, any derivative of $f_p(x)$ not involving $\frac{\partial}{\partial w..^\ell}$ vanishes except $\frac{\partial}{\partial {b.^\ell}^2}$.  Therefore with the exception of the derivative $\frac{\partial^2}{\partial {b.^\ell}^2}$, both terms of \ref{twoderivatives} vanish for any second derivative of $L$ not involving $\frac{\partial}{\partial w..^\ell}$.  

The derivatives of $L$ it remains to check are $\frac{\partial^2}{\partial {b.^\ell}^2}$ and $\frac{\partial}{\partial *}\frac{\partial}{\partial w..^\ell}$.  Again we use the expression \ref{fentry}
$$
\sum w..^\ell \sigma \left( \sum w..^{\ell-1} ... \sigma \left( \sum w..^2 \sigma \left( \sum w..^1 z. + b.^1 \right) + b.^2 \right) + ... + b.^{\ell-1} \right) + b.^\ell.
$$

First, we compute $$\frac{\partial^2}{\partial {b.^\ell}^2} L(p).$$ 

The derivative $\frac{\partial^2}{\partial {b.^\ell}^2} L(p) = 2n$, in particular is positive and nonzero.

Next, we compute $$\frac{\partial}{\partial *}\frac{\partial}{\partial w..^\ell} L(p).$$

Well,
\begin{align} \label{twoderivativesw}
\frac{\partial}{\partial *} \frac{\partial}{\partial w..^\ell} L(p) = 2 \sum_{\alpha=1}^n \frac{\partial}{\partial *} f_p(x_\alpha) \frac{\partial}{\partial w..^\ell} f_p(x_\alpha) +(f_p(x_\alpha) - y_\alpha) \frac{\partial}{\partial *} \frac{\partial}{\partial w..^\ell} f_p(x_\alpha).
\end{align}

We may check that at any $p \in C$, 
$$\frac{\partial}{\partial w..^\ell} f_p(x) = 0.$$
So the first term of \ref{twoderivativesw} vanishes.  

For the second term, we consider two cases. 

Case 1: $* = w..^i$.

The coordinates of the derivative
$$\frac{\partial}{\partial w..^i} \frac{\partial}{\partial w..^\ell} f_p(z)$$
have the form
\begin{align*}
& \sigma \left( \sum w..^{i - 1} \sigma \left( ... \right) + b.^{i-1} \right) \\
&  \cdot w..^{i+1} \sigma' \left( \sum w..^i \sigma \left( ... \right) + b.^i  \right) \\
& \cdot w..^{i+2} \cdot \sigma' \left( \sum w..^{i+1} \sigma \left( ... \right) + b.^{i+1}  \right) \\ 
& \hspace{.3in}\vdots \\ 
& \cdot w..^\ell \cdot \sigma' \left( \sum w..^{\ell-1} \sigma \left( ... \right) + b.^{\ell-1}  \right). \\
\end{align*}

for every point $p \in C$, this expression is 0, so both terms of \ref{twoderivativesw} vanish.

Case 2: $* = b.^i$.

The coordinates of the derivative
$$\frac{\partial}{\partial b..^i} \frac{\partial}{\partial w..^\ell} f_p(z)$$
have the form
\begin{align*}
&  \cdot w..^{i+1} \sigma' \left( \sum w..^i \sigma \left( ... \right) + b.^i  \right) \\
& \cdot w..^{i+2} \cdot \sigma' \left( \sum w..^{i+1} \sigma \left( ... \right) + b.^{i+1}  \right) \\ 
& \hspace{.3in}\vdots \\ 
& \cdot w..^\ell \cdot \sigma' \left( \sum w..^{\ell-1} \sigma \left( ... \right) + b.^{\ell-1}  \right). \\
\end{align*}

for any point $p \in C$, let $j$ be the largest integer, except for $\ell$, for which $w..^j = 0$.  If $j > i$, this expression is 0, so both terms of \ref{twoderivativesw} vanish.  If $j \leq i$, then this expression has no dependence on $x_\alpha$, so the second term of \ref{twoderivativesw} becomes

$$
2 \sum_{\alpha=1}^n (f_p(x_\alpha) - y_\alpha) \frac{\partial}{\partial *} \frac{\partial}{\partial w..^\ell} f_p(x_\alpha) = v \cdot 2 \sum_{\alpha=1}^n (f_p(x_\alpha) - y_\alpha)
$$
for some vector $v$ independent of $\alpha$.  But
$$
\sum_{\alpha=1}^n (f_p(x_\alpha) - y_\alpha) = 0 
$$
so both terms of \ref{twoderivativesw} vanish.

We conclude that every second derivative of $L(p)$ vanishes except for $\frac{\partial^2}{\partial {b.^\ell}^2} L = 2n$.  
\end{proof}

\subsection{Attracting critical loci}
Gradient flow on a smooth function $L$ is guaranteed to converge to a critical point or diverge to $\infty$.  For feedforward networks with $L2$ loss, it is expected that gradient flow does not diverge to $\infty$, which leaves the question of which critical points are reached by gradient flow under random initialization.

A classical concept is that of the stable set of a critical point $p$.  This is defined as the set of all points $q$ such that gradient flow initialized at $q$ converges to $p$.  For example, given an isolated nondegenerate saddle point $p$, the stable set of $p$ is zero measure.  This fact is central to the proof in \cite{jordansaddle} that gradient based methods can efficiently escape nondegenerate saddles.  

The critical loci that are difficult for gradient based methods to escape are those with positive measure stable sets.  Hence we make the following definition.

\begin{definition}
A locus $C$ of critical points is called an attracting critical locus if there exists a positive measure set $S_C$ such that gradient flow initialized at any point in $S_C$ converges to a point in $C$.
\end{definition}

Note that for $C$ to be an attracting critical locus, it is not necessary for the stable manifold of every point $p$ in $C$ to have full measure.  For example, the $x$-axis is an attracting critical locus for the function $f(x,y) = x^2$.

It is well understood which loci of nondegenerate critical points can be attracting critical loci.  Namely, loci of nondegenerate local and global minima can be attracting critical loci.  Loci of nondegenerate local and global maxima cannot.  Isolated nondegenerate saddles cannot be attracting critical points.  However, it is not well understood in general which degenerate critical points can form attracting critical loci.  As in the nondegenerate case, loci of degenerate local and global minima can still be attracting critical loci.  However, in the degenerate case, saddles can also be attracting critical points.  

Not only can an isolated degenerate saddle point be an attracting critical point, as we will see in the example that follows, but if there is a positive dimensional locus of degenerate saddle points, that locus can form an attracting set even if individually none of the points are attracting critical points.  Given the presence of positive dimensional families of degenerate critical points of the loss function $L$, it is of interest to understand which degenerate critical points of $L$ can form attracting critical loci.  We do not know at present of a proof that the core locus $C_2$ is not an attracting critical locus.  

\begin{example*}
As an example of a degenerate critical point which is an attracting critical point, consider the function 
$$f(x,y) = x^3 + y^3.$$  

\begin{figure}\label{x3+y3}
\begin{center}
\includegraphics[width=3.5in]{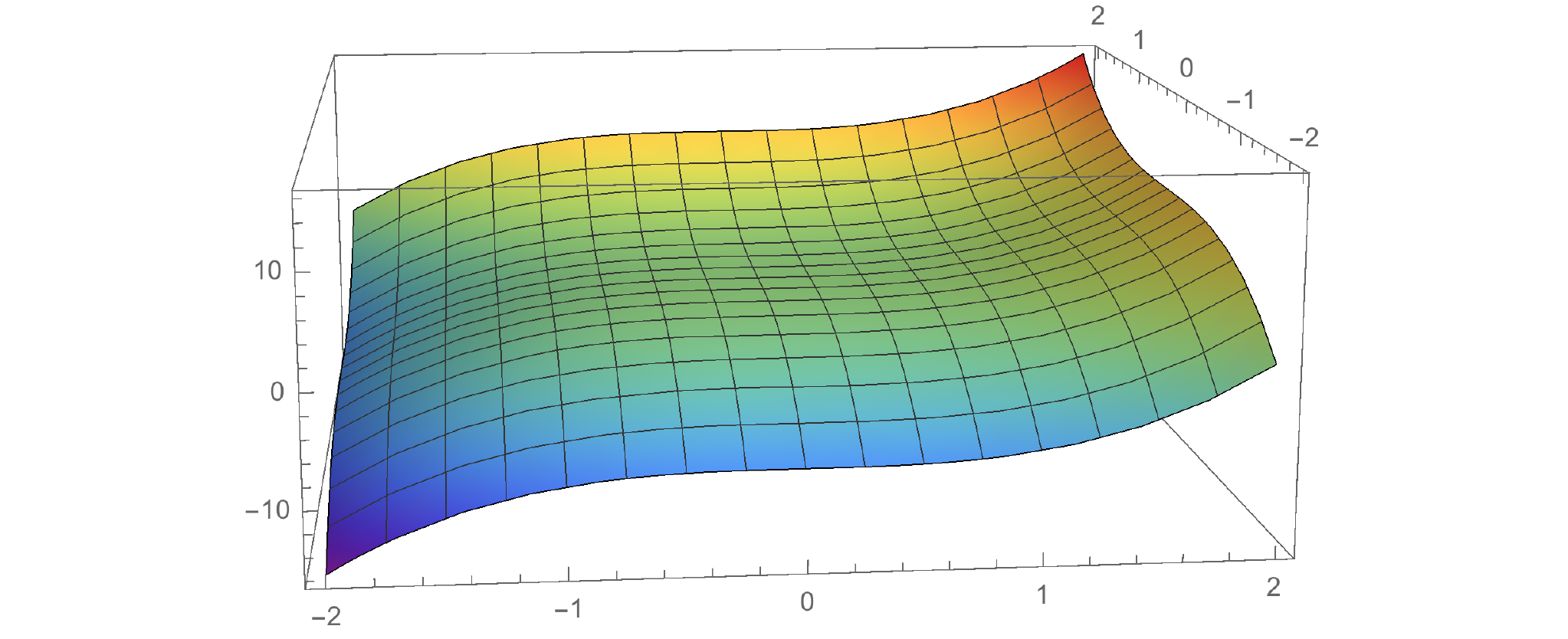}
\end{center}
\caption{The function $f(x,y) = x^3 + y^3$ has a degenerate critical point at the origin, and this critical point is an example of an attracting critical point, as a set of positive measure flows to it under gradient flow.}
\end{figure}

The graph of $f$ is shown in Figure \ref{x3+y3}.  The origin is an isolated critical point, and this critical point is degenerate --- the Hessian of $f$ vanishes at the origin.  Every point in the quadrant $(x,y)$ with $x, y$ both nonnegative, flows to the origin under gradient flow.  Hence the origin is an attracting critical point.
\end{example*}

\section{Mildly overparameterized regime}\label{overparam}
In this section, we recall some basic facts about the loss function $L$ that hold in all three overparameterized regimes.  We know less here than in the very wide and extremely wide settings, but this regime is the most interesting because it is exactly the regime in which real-world neural networks usually lie.  Although this is the theoretically most challenging of the overparameterized regimes, in previous work we were able to establish the following basic understanding about the locus of global minima of $L$ in this setting.  

\subsection{Global minima}

In \cite{cooper}, we showed that for any overparameterized neural network with $d$ parameters, smooth activation function $\sigma$, L2 loss $L(p)$, and training on $n$ distinct data points $\{(x_i, y_i)\}$, where $x_i \in \R^a$ and $y_i \in \R^b$,

\begin{theorem*}
If $d>n$ then the set $M = L^{-1}(0)$ is generically (that is, possibly after an arbitrarily small change to the data set) a smooth $d-bn$ dimensional submanifold (possibly empty) of $\R^d$.
\end{theorem*}

We further considered the case that the width of the last hidden layer $m_{\ell-1}$ is greater than $n$, which is essentially the very wide regime.  In this case, we showed by construction that $M$ is nonempty.  However, we expect that $M$ is nonempty long before the width of the last hidden layer reaches $n$.  

\section{Very wide regime}\label{verywide}
Now we record some facts about the geometry of $L$ that hold in the very wide and extremely wide regimes, that is, that hold as soon as $m \geq n$.  

Several groups have made substantial progress in uncovering the geometry of the loss function $L$ in the case of very wide feedforward neural networks.  Several groups, including \cite{fullrank}, \cite{quynhconnectedsublevel}, \cite{ruoyumeasure0}, and \cite{venturibruna} show that in this setting, the loss function $L$ has no spurious valleys.  Neither paper rules out the possibility that $L$ contains local minima, but both show that if $p$ is a local minimum, then the connected component of the level set of $L$ containing $p$ is positive dimensional.  In this section, we build on their results, toward a more quantitative understanding of the local geometry of $L$ at different kinds of critical points, for feedforward neural networks in the very wide regime.  

For completeness, we record here some standard definitions.

\begin{definition} \label{localmin}
A continuous function $L$ is said to have a local minimum at the point $p$ if there exists some $\epsilon > 0$ such that $L(q) \geq L(p)$ for all $|q-p| < \epsilon$.
\end{definition}

\begin{definition}[\cite{venturibruna}] \label{spuriousvalley} 
A continuous function $L$ is said to have a spurious valley if there is some sublevel set $L^{-1}\bigl((-\infty,a]\bigr)$ and a component $E$ of that sublevel set on which $\inf (L|_E) \neq \inf L.$
\end{definition}

\begin{definition}[\cite{fullrank}] \label{badvalley} 
A continuous function $L$ is said to have a bad local valley if there is some strict sublevel set $L^{-1}\bigl((-\infty,a)\bigr)$ and a component $E$ of that strict sublevel set on which $\inf (L|_E) \neq \inf L.$
\end{definition}

\begin{definition}\label{strictlocalmin}
A local minimum $p$ is called a strict local minimum if there is an $\epsilon$-neighborhood of $p$ such that for every $q \in B_{\epsilon}(p) \setminus{p} $, $L(q)>L(p)$.  
\end{definition}

\begin{definition} \label{isolatedcritpoint}
A critical point $p$ is called isolated if there exists some $\epsilon>0$ such that there is no other critical point within an $\epsilon$-ball of $p$.  Similarly, a local minimum/local maximum/saddle point $p$ is called isolated if there exists some $\epsilon>0$ such that there is no other local minimum/local maximum/saddle point within an $\epsilon$-ball of $p$
\end{definition}

\begin{definition} \label{nondegencritpoint}
A critical point $p$ is a nondegenerate critical point if every eigenvalue of the Hessian of $L$ at $p$ is nonzero.  
\end{definition}

We also include the Morse lemma, which is a description of fundamental importance of the local behavior of a function near a nondegenerate critical point.

\medskip

\noindent
{\bf Morse Lemma (\cite[Thm.~2.39]{lee}).}
{\it Let $f\colon M \to \mathbb R$ be a smooth function and let $x_0$ be a nondegenerate critical point for $f$ of index $\nu$.
Then there is a local coordinate system $(U,\mathbf x)$ containing $x_0$ such that the local representative $f_U := f \circ \mathbf x^{-1}$ has the form
\begin{align*}
f_U(x^1,\ldots,x^n)
=
f(x_0)
+
\sum_{i,j} h_{ij}x^ix^j
\end{align*}
and it may be arranged that the matrix $h = (h_{ij})$ is a diagonal matrix of the form $\text{diag}(-1,\ldots,-1,1,\ldots,1)$ for some number (perhaps zero) of ones and minus ones.
The number of minus ones is exactly the index $\nu$.}

\medskip

\subsection{$R$-map}

We now fix some notation for the maps that will be relevant for us in this section.
For each layer, we define the map 
\begin{equation*}
\phi_i: \R^{d_1 + ... + d_i} \times \R^a \rightarrow \R^{m_i}
\end{equation*}
by
\begin{equation*}
\phi_i(p_1, ..., p_i, x) = \sigma \circ A_{p_i} \circ ... \circ \sigma \circ A_{p_1}(x).
\end{equation*}
In other words, $\phi_i$ computes the function determined by the parameters $p_1,...,p_i$ of the first $i$ layers of the neural network.

It will be important for us to understand not just the output of the $i^{th}$ layer given a single input vector $x$, but to understand the output of the $i^{th}$ layer on all input data vectors $x_1, ..., x_n$ simultaneously.  So we now define the function $\Phi_{i}: \R^{d_1 + ... + d_i} \times \R^{a n} \rightarrow \R^{m_i n}$ by
$$\Phi_{i}(p_1, ..., p_i , x_1, ..., x_n) = 
\left(
\begin{bmatrix}
| \\ \phi_{i}(p_1, ..., p_i, x_1) \\ |
\end{bmatrix}
\dots
\begin{bmatrix}
| \\ \phi_{i}(p_1, ..., p_i, x_n) \\ |
\end{bmatrix}
\right).
$$

Given a choice of parameters $p=(p_1, ..., p_\ell)$, our neural network computes the function 
$$f_{p}(x) = A_{p_\ell} \circ \sigma \circ A_{p_{\ell-1}} \circ ... \circ \sigma \circ A_{p_1}(x)$$ 
where $A_{p_i}$ is the affine linear transformation from $\R^{m_{i-1}}$ to $\R^{m_i}$ determined by the parameter vector $p_i$.  
 
In particular, $f_p$ is a composition of a string of functions, and it is often useful to decompose this into the composition of two functions.  

$$f_p(x) = A_{p_\ell} \circ \left(\sigma \circ A_{p_{\ell-1}} \circ ... \circ \sigma \circ A_{p_1}(x) \right)= R \circ Q(x)$$ 
where we name the last function $R = A_{p_\ell}$ and $Q$ is the composition of all of the preceding functions, that is
$Q =  \sigma \circ A_{p_{\ell-1}} \circ ... \circ \sigma \circ A_{p_1}.$  

The space $\R^d$ of all parameters for the neural network is a product of the spaces $\R^{d_i}$ of the parameters for each layer
$$\R^d = \R^{d_1} \times ... \times \R^{d_\ell}$$
where $d_i = m_i (m_{i-1}+1)$.

When decomposing the neural network as the composition of two functions, it is useful to also decompose in a compatible way this parameter space into a product of two spaces.  Let $\R^I = \R^{d_1} \times ... \times \R^{d_{\ell-1}}$ denote the parameters corresponding to all but the final layer, and $\R^F = \R^{d_\ell}$ denote the parameters corresponding to the final layer.  Then 
$$\R^d = \R^I \times \R^F.$$  

The loss function $L$, which is a function from $\R^d$ to $\R$, can also be considered as a function from $\R^I \times \R^F$ to $\R$.  In particular, we can consider the restriction of $L$ to special subsets of $\R^I \times \R^F$.  

The geometry of $L$ as a function on $\R^d$ can be very complicated.  However, consider the projection 
$$\pi: \R^I \times \R^F \rightarrow \R^I.$$

It turns out that the restriction of $L$ to any fiber of $\pi$ is very simple.  The following two statements Proposition \ref{quadratic} and Lemma \ref{fullrankglobalmin} appear in some guise in multiple works, including \cite{quynhconnectedsublevel} and \cite{ruoyumin}, but we include them here under our assumptions for completeness.

\begin{proposition} \label{quadratic}
For any $p_I \in \R^I$, $L$ restricted to the fiber $S_{p_I} = \pi^{-1}(p_I)$ is a quadratic function.  
\end{proposition}

\begin{proof}
\begin{align*}
L(p) &= L(p_I, p_F) \\
&= \sum_{i=1}^n \left( f_{(p_I, p_F)} (x_i) - y_i \right)^2 \\
&= \sum_{i=1}^n \left( M_{p_\ell} \phi_{\ell-1}(p_I, x_i) + b_{p_\ell} - y_i \right)^2 \\
\end{align*}
In the fiber $S_{p_I}$, both $p_I$ and all the $x_i$ are fixed, so the vectors $\phi_{\ell-1}(p_I, x_i)$ are constant.  Thus as a function of $w..^{p_\ell}$ and $b.^{p_\ell}$, 
$$
L|_{\pi^{-1}(p_I)}
$$
is a quadratic function.
\end{proof}

\begin{figure}\label{slices}
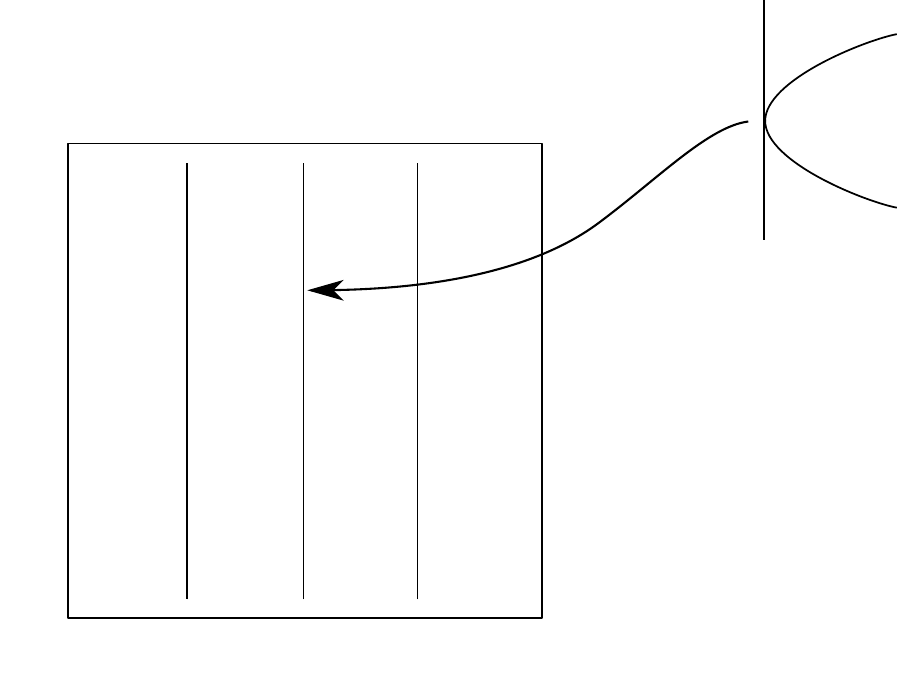
\end{figure}

This gives us some handle on the geometry of $L$, as most functions $G: \R^d \rightarrow \R$ do not satisfy the property that the domain can be sliced in such a way that $G$ restricted to each slice is a quadratic function.  We now analyze the geometry of the critical points of $L$ by studying $L$ in each slice.  

\begin{lemma} \label{fullrankglobalmin}
If $p_I$ is such that $\Phi_{\ell-1}(p_I, x_1, ..., x_n)$ has full rank, then any critical point in $\pi^{-1}(p_I)$ is a global minimum of $L$.
\end{lemma}
\begin{proof}
By Proposition \ref{quadratic}, $L|_{\pi^{-1}(p_I)}$ is a quadratic function.  Furthermore, $L$ is nonnegative, so $L|_{\pi^{-1}(p_I)}$ is also nonnegative.  The only critical points of a nonnegative quadratic function are global minima.  If $p \in \pi^{-1}(p_I)$ is a critical point of $L$, it must be a critical point also of $L|_{\pi^{-1}(p_I)}$, and hence a global minimum of $L|_{\pi^{-1}(p_I)}$.  

If $\Phi_{\ell-1}(p_I, x_1, ..., x_n)$ has full rank, then there is some matrix $M$ which satisfies the equation
$$
M \phi_{\ell-1}(p_I, x_i) = y_i
$$
for each $i = 1,...,n$.  

Hence there is some choice of parameters $p_F$ such that 
$$
\sum_{i=1}^n \left(M_{p_F} \phi_{\ell-1}(p_I, x_i) - y_i \right)^2 = 0
$$
i.e.
$$
L(p_I, p_F) = 0.
$$

Thus the global minimum of $L|_{\pi^{-1}(p_I)}$ is 0.  Since we deduced that $p$ must be a global minimum of $L|_{\pi^{-1}(p_I)}$, $L(p)$ must be 0, and hence $p$ is in fact a global minimum of $L$ as well. 
\end{proof}

As a corollary of Proposition \ref{quadratic} and Lemma \ref{fullrankglobalmin}, we can prove the following.  Similar results have been proved under somewhat different assumptions in \cite{fullrank} and \cite{ruoyumeasure0}.

\begin{corollary} \label{omega}
Consider a fully connected feedforward network with $L2$ loss $L$, smooth activation function $\sigma$, and last hidden layer of width $m_{\ell-1} \geq n$ where $n$ is the number of training samples.  Let $\Omega$ denote the set of parameter vectors $p_I \in \R^I$ with the property that $\Phi_{\ell-1}(p_I, x_1, ..., x_n)$ has full rank.  Then for every $p \in \Omega$ that is not a global minimum of $L$, there exists a line $C_p$ from $p$ to some global minimum $g_p$ such that $L$ is strictly decreasing along $C_p$.  Furthermore, $L$ decreases quadratically along $C_p$.  
\end{corollary}

\begin{proof}
By Proposition \ref{quadratic}, $L$ restricted to $p_I \times \R^F$ is a nonnegative quadratic function.  Hence the locus of minima in this fiber is a linear subspace and $L$ restricted to any line $C_p$ from $p$ to a point $g_p$ in this linear subspace is also a quadratic function, and $L$ is strictly decreasing as one moves from $p$ to $g_p$.

Finally, by Lemma \ref{fullrankglobalmin}, any minimum of $L|_{\pi^{-1}(p_I)}$ is a global minimum of $L$, hence $g_p$ is a global minimum of $L$.
\end{proof}

\subsection{Measure zero}

Given Corollary \ref{omega}, it is of interest to know whether $\Omega \times \R^F \subset \R^d$ contains nearly all parameter vectors.  In \cite{fullrank}, Nguyen et. al. prove that this is indeed the case, when the activation function $\sigma$ is an analytic, strictly increasing, bounded function with  $\lim\limits_{t \rightarrow \infty} \sigma(t) = 0$ and all input data points $x_i$ distinct.  They work with a large class of architectures, allowing convolutional networks as well as skip connections, but in this paper we will only consider the feedforward setting.  

\begin{theorem}[\cite{fullrank}]\label{openanddense} 
Given $n$ distinct vectors $x_1, ..., x_n \in \R^a$, and activation function $\sigma$ satisfying the above assumptions, the complement of the region $\Omega \subset \R^{d_1 + ... + d_{\ell-1}}$ of parameter vectors $p$ such that $\Phi_{\ell-1}(p, x_1, ..., x_n)$ has full rank has measure zero.
\end{theorem}

\begin{theorem}
Consider a fully connected feedforward network with $L2$ loss and an activation function $\sigma$ which is an analytic, strictly increasing, bounded function with $\lim\limits_{t \rightarrow \infty} \sigma(t) = 0$, and each layer of width greater than $n$ where $n$ is the number of training samples, and all input data points $x_i$ distinct.  Then there is a set $\Omega$ within the parameter space $\R^d$ such that the measure of $\R^d \setminus \Omega$ is zero and for any point $p \in \Omega$, there is a line $C_p$ from $p$ to some global minimum $g_p$ such that $L$ is strictly decreasing along $C_p$.
\end{theorem}

\begin{proof}
This follows directly from Corollary \ref{omega} and Theorem \ref{openanddense}.
\end{proof}

In \cite{ruoyumeasure0}, Sun et. al. prove a similar analogous to Theorem \ref{openanddense}, for feedforward networks under different assumptions.  Their assumptions allow for a different class of activation functions, namely those which are analytic and have at least $n$ nonvanishing derivatives at the origin.  This allows such commonly used activation functions as tanh and swish.  In their setting, they prove that the complement of the set $\Omega$ where $\Phi_{\ell-1}(p, x_1, ..., x_n)$ has measure 0.  

\subsection{All critical points which are not global minima}

In this section, we are interested in the local geometry of $L$ near critical points $p$ that are not global minima.  An important case is when $p$ is a local minimum.  A simple example of a local nonglobal minimum is the origin, for the function $f(x) = x^4 + 5 x^3 + 6 x^2$.  There, the origin $p$ is an isolated local minimum, a strict local minimum, and lies in a spurious valley.  (Recall these notions were defined in Definitions \ref{spuriousvalley} -- \ref{isolatedcritpoint}.)

In general, these three properties do not always coincide.  We begin here by untangling the relationships between these properties and discuss which hold for local minima in very wide neural networks.  The relationships are summarized in Figure \ref{webofproperties}.

\begin{figure}
\begin{center}
\includegraphics[width=3.5in]{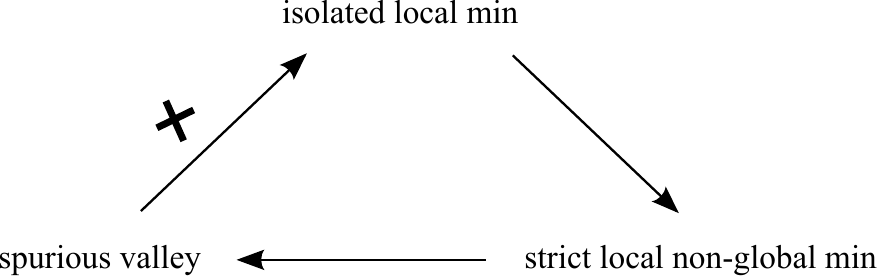}
\end{center}
\caption{Implications between existence of isolated local minima, strict local minima, and spurious valleys.
\label{webofproperties}}
\end{figure}

Let us consider the sides of the triangle in Figure \ref{webofproperties}.  One leg is proved in \cite{fullrank}, where Nguyen, Mukkamala, and Hein show that if a function has no spurious valleys then it has no strict local nonglobal minima.  One leg does not hold - it is not true that if $L$ has a spurious valley then it must have an isolated local minimum, and the function $f(x,y) = 5 x^4 - 5 x^2 + x$ gives an example where that implication doesn't hold.  On the last leg of the triangle, several authors \cite{numericaloptimization}, \cite{nonlinearprogramming}, \cite{dun}, have stated that if $p$ is an isolated local minimum then $p$ is a strict local minimum, but we do not know of a proof in the literature, so we include one here, based on an argument suggested by Ruoyu Sun.

\begin{lemma} \label{NOstrictNOisolated}
Suppose $L\colon \R^d \to \R$ is a continuous function with no strict local, non-global minima.
Then $L$ has no isolated local, non-global minima.
\end{lemma}

\begin{proof}
We will prove this by showing that if $p$ is a local, non-global minimum that is not strict, then $p$ is not an isolated minimum.
Fix such a $p$.
Since $p$ is not strict, we may choose a sequence $q_k$ such that $|q_k - p| \leq 2^{-k}$ and $L(q_k) = L(p)$ for all $k$.
I claim that for $k$ large enough, $q_k$ is a local minimum.
Indeed, suppose that this is not the case.
Then for every $k$, there is a $q_k'$ with $|q_k' - q_k| \leq 2^{-k}$ and $L(q_k') < L(q_k)$.
We then have $|q_k' - p| \leq |q_k' - q_k| + |q_k - p| \leq 2^{-k+1} \to 0$, and $L(q_k') < L(q_k) = L(p)$, hence $L(p)$ is not a local minimum, which is a contradiction.
This shows that $p$ must indeed not be isolated.

We have now shown that if $p$ is a local, non-global local minimum that is not strict, then $p$ must not be isolated.
Since this applies to all such $p$, the statement of the lemma must hold.
\end{proof}

As mentioned at the beginning of this section, under various assumptions, both \cite{fullrank} and \cite{venturibruna} show that for very wide feedforward neural networks the loss function $L$ has no spurious valleys, and \cite{ruoyumeasure0} shows the related result that $L$ has no strict local minima.  It is possible to extend each of these results to the statement that under the appropriate assumptions, $L$ has no isolated local non-global minima.  Here, we do this for the Theorem of \cite{ruoyumeasure0}.    

\begin{proposition} \label{noisolatedmin}
For a fully connected feedforward neural network with width $m$ larger than the number of training samples $n$, continuous activation function $\sigma$ and L2 loss function $L$, $L$ has no isolated local non-global minima.
\end{proposition}

\begin{proof}
Under these assumptions  Theorem 1 of \cite{ruoyumeasure0} proves that $L$ has no strict local minima.  (In fact Theorem 1 of \cite{ruoyumeasure0} holds even in a more general setting that what we're studying here.)  By Lemma \ref{NOstrictNOisolated}, we conclude that $L$ has no isolated local minima.
\end{proof}

The analogous statements hold for local maxima.  This kind of argument does not however hold for other critical points, and in fact, a simple nondegenerate saddle point provides an example of an isolated critical point which lies on a positive dimensional level set.  Namely, if we let
$$f(x,y) = x^2-y^2,$$
the origin is an isolated critical point, but this critical point lies on a level set which is the union of the lines $x=y$ and $x=-y$.

In this section, we aim to extend the result that for very wide networks $L$ has no spurious valleys to a more detailed analysis of the local geometry of $L$ near its critical points.  Suppose $p=(p_I, p_F)$ is a critical point of $L$ but is not a global minimum of $L$.  

To start with, we observe that by Lemma \ref{fullrankglobalmin}, $\Phi_{\ell-1}(p_I, x_1, ..., x_n)$ cannot have full rank.  Next, we note that in each subspace corresponding to a vertical slice in Figure \ref{slices}, the loss function $L$ is a quadratic function.  Furthermore, $L$ is a nonnegative function. Therefore for any fixed slice $P$, $L|_P$ is a nonnegative quadratic function.  Let $Q|_P$ denote the quadratic form given by the leading order terms.  That is, $Q|_P$ is $L|_P$ minus all linear and constant terms.  Since $Q|_P$ is also nonnegative, it is a positive semi-definite quadratic form. 

Note that for any semi-definite quadratic form, all local minima are global, and the locus of global minima is connected, in particular, is a linear subspace whose codimension is the rank of the quadratic form.  So we begin by computing the rank of the semi-definite quadratic form $Q|_P$.  

\begin{lemma}
Suppose that $m_{\ell-1} > n, m_\ell$ and that in a fixed slice $p_I$, the rank of $\Phi_{\ell-1}(p_I, x_1, ..., x_n) = r$.  Then in that slice, the rank of the positive semi-definite quadratic form $Q|_{p_I}$ is less than or equal to $r m_\ell$.  
\end{lemma}

\begin{proof}
First we consider the case of a feedforward neural network trained on a single data point $(x,y)$.

Let 
$$A = 
\begin{pmatrix}
a_{11} &   & a_{1 m_{\ell-1}} \\
\vdots & \hdots & \vdots \\
a_{m_\ell 1} &  & a_{m_\ell m_{\ell-1}} \\
\end{pmatrix}
$$
and
$$b = 
\begin{pmatrix}
b_1 \\
\vdots \\
b_{m_\ell}
\end{pmatrix}
$$
encode the affine linear transformation from the last hidden layer of the network to the output, that is a map from $\R^{m_{\ell-1}}$ to $\R^{m_\ell}$.  

Let
$$y = 
\begin{pmatrix}
y_1 \\
\vdots \\
y_{m_\ell}
\end{pmatrix}
$$
denote the output of the single data point, and let
$$\xi = 
\begin{pmatrix}
\xi_1 \\
\vdots \\
\xi_{m_{\ell-1}}
\end{pmatrix}
$$
denote the output of the last hidden layer, that is 
$$
\xi = \phi_{\ell-1}(p_1, ..., p_{\ell-1}, x).
$$

Fixing a slice $P$ is equivalent to fixing $(p_1, ..., p_{\ell-1})$, so we treat $\xi$ as a constant in this proof.

The quadratic function $L|_P$ is then
$$|M^\ell \xi + b^\ell - y|^2$$
and the quadratic form $Q|_P$ is
$$|M^\ell \xi |^2.$$

We compute the rank of this quadratic form, considered as a function from $\R^{m_{\ell-1} m_\ell}$ to $\R$.  That is, with the entries of $M^\ell$ 
considered as variables, and the entries of $\xi$ considered as constants.  

Well,
$$ |M^\ell \xi |^2 = 
\begin{pmatrix}
v_1 \\
\vdots \\
v_{m_\ell} \\
\end{pmatrix} ^2 
= v_1^2 + ... + v_{m_\ell}^2
$$
where
$$v_i = a_{i 1} \xi_1 + ... + a_{i m_{\ell-1}} \xi_{m_{\ell-1}}.$$

The rank of each quadratic form $v_i^2$ is one.  Note that the variables that appear in $v_i$ are distinct from those in $v_j$, for all $i \neq j$, hence the rank of the quadratic form $v_1^2 + ... + v_{m_\ell}^2$ is exactly $m_\ell$.  

Finally, we consider the case of $n$ data points, where the rank of the matrix 
$$
\Phi_{\ell-1}(p_I, x_1, ..., x_n) = 
\begin{pmatrix}
\xi_1, ..., \xi_n
\end{pmatrix}
$$ 
is $r$.  

Without loss of generality, suppose the first $r$ columns of $\Phi_{\ell-1}(p_I, x_1, ..., x_n)$ are linearly independent.  Then we can write the matrix $\Phi_{\ell-1}(p_I, x_1, ..., x_n)$ as
$$
\Phi_{\ell-1}(p_I, x_1, ..., x_n) = 
\begin{pmatrix}
\xi_1, ..., \xi_r, \chi_1, ..., \chi_{n-r}
\end{pmatrix}
$$ 
where each $\chi_i$ is a linear combination of $\xi_1, ..., \xi_r$.

In this case, $Q|_P$ is of the form
\begin{align*}
& \left( c_1\cdot \xi_1\right)^2 + ... + \left( c_{m_\ell} \cdot \xi_1\right)^2 + ... \\
...\\
& \left( c_1\cdot \xi_r\right)^2 + ... + \left( c_{m_\ell} \cdot \xi_r\right)^2 + ... \\
& \left( c_1 \cdot \chi_{1} \right)^2 + ...  + \left( c_{m_\ell} \cdot \chi_{1}\right)^2 \\
...\
& \left( c_1 \cdot \chi_{n-r} \right)^2 + ...  + \left( c_{m_\ell} \cdot \chi_{n-r}\right)^2 \\
\end{align*}
where
$$c_i = 
\begin{pmatrix}
a_{i 1} \\
\vdots\\
a_{i m_{\ell-1}} 
\end{pmatrix}.
$$

Such a quadratic form has rank at most $r m_\ell$, completing the proof.
\end{proof}

As a result, we can deduce the following.  

\begin{theorem}  \label{dimlin}
For any critical point $p=(p_I, p_F)$ of $L$, the level set of $p$ contains a linear subspace of dimension $(m_{\ell-1}+1 - r) m_\ell$, where $r = rank \Phi_{\ell-1}(p_I, x_1, ..., x_n)$.  
\end{theorem}
\begin{proof}
If $p$ is a critical point, it means that $\Phi_{\ell-1}(p_1, ..., p_{\ell-1}, x_1, ..., x_n)$ is not full rank.  Let us decompose $(p_1, ..., p_{\ell-1}, p_\ell) = (p_I, p_F)$

Then in the slice determined by $p_I$, $Q|_{\pi^{-1}(p_I)}$ is a semi-definite quadratic form of rank at most $r m_\ell$.  Moreover, the dimension of the locus of global minima of $L|_{\pi^{-1}(p_I)}$ in the slice is equal to the dimension of the locus of global minima of $Q|_{\pi^{-1}(p_I)}$.  
That means that there is a linear subspace $S$ of codimension at most $r m_\ell$ in $\R^F = \R^{(m_{\ell-1}+1) m_\ell}$ such that $S$ is in the level set containing $p$.  This means that 
$$\dim(S) \geq (m_{\ell-1} + 1 - r) m_\ell,$$
which completes the proof.
\end{proof}

Note this result implies that $(m_{\ell-1}+1 - r) m_\ell$ is a lower bound on the number of zero eigenvalues of the Hessian of $L$ at $p$, which we record as a corollary now.  However, Theorem \ref{dimlin} is stronger, in that it gives us some partial knowledge of the global geometry of $L$, not just the geometry of $L$ in a neighborhood of $p$.  

\begin{corollary}\label{critzeroeigen}
For any critical point $p$ of $L$, $Hess(L)$ at $p$ has at least $(m_{\ell-1}+1 - r) m_\ell$ zero eigenvalues, where $r = rank \Phi_{\ell-1}(p_I, x_1, ..., x_n)$.
\end{corollary}
\begin{proof}
By Theorem \ref{dimlin}, the level set of $p$ contains a linear subspace $S$ of dimension $k=(m_{\ell-1}+1 - r) m_\ell$.  Choose local coordinates $s_1, ..., s_k, t_1, ..., t_{d-k}$ at $p$ such that $s_1, ..., s_k$ spans the linear subspace $S$.  Then computing the Hessian of $L$ in these coordinates, we get a $k \times k$ block of zeroes, hence the Hessian of $L$ has at least $k$ zero eigenvalues.  
\end{proof}

Since the rank $r$ of $\Phi_{\ell-1}(p_I, x_1, ..., x_n)$ can never be greater than $n$, we can also state the following weaker forms of Theorem \ref{dimlin} and Corollary \ref{critzeroeigen} that hold uniformly for all critical points of $L$.  
\begin{theorem}  \label{dimlinuniform}
For any critical point $p$ of $L$, the level set of $p$ contains a linear subspace of dimension $(m_{\ell-1}+1 - n) m_\ell$.  
\end{theorem}

\begin{corollary}\label{critzeroeigenuniform}
For any critical point $p$ of $L$, $Hess(L)$ at $p$ has at least $(m_{\ell-1}+1 - n) m_\ell$ zero eigenvalues.
\end{corollary}

When $p$ is a local maximum, Theorem \ref{dimlin} implies the following corollary.

\begin{corollary}  \label{dimlinformax}
For any local maximum $p$ of $L$, $p$ the level set of $p$ contains a linear space of dimension $m_\ell(m_{\ell-1}+1)$.  
\end{corollary}

\begin{proof}
Suppose $p = (p_I, p_F) \in \R^I \times \R^F$ is a global maximum.  Then the rank of $\Phi_{\ell-1}(p_I, x_1, ..., x_n)$ is zero.  Hence the level set containing $(p_I, p_F)$ contains a linear space of dimension $m_\ell(m_{\ell-1}+1)$, in particular contains the entire slice $\pi^{-1}(p_I)$. 
\end{proof}

\section{Discussion}\label{discussion}

\subsection{Comparison of critical points}

Let us compare the estimates for the different critical loci discussed in this paper in the case that all the hidden layers have the same width $m_1 = ... = m_{\ell-1} = m$, and letting $a=m_0, b= m_\ell$ for clearer formulas.

In this case, the dimension of the parameter space is
$$
d = \dim(P) = (\ell-2) m^2 + (\ell+a+b-1) m + b.
$$

The star locus $S$ always has positive dimension for any $m,n$, as long as $\ell \geq 3$.
$$
\dim(S) = d- \left(m^2 + (\ell + b -2) m + b\right) = (\ell-3)m^2 + (a+1)m.
$$

The core locus $C$ also always has positive dimension for any $m,n$, as long as $\ell \geq 4$.
$$
\dim(C) = (\ell-4)m^2 + (a+1)m.
$$

Furthermore, the eigenvalues of $Hess(L)$ at any point $p \in C$ are all zero except one positive eigenvalue.

Meanwhile, as soon as we enter the overparameterized regime and $d>n$, the locus of parameters  $M=L^{-1}(0)$ that fit the training data perfectly, if nonempty, has dimension
$$
\dim(M) = (\ell-2)m^2 + (a+b+\ell-1)m + b(1-n)
$$
and at each point $p \in M$, the Hessian of $L$ at $p$ has at least $(\ell-2)m^2 + (a+b+\ell-1)m + b(1-n)$ zero eigenvalues.  

Finally, as soon as we enter the very overparameterized regime with $m>n$, the level set for any critical point contains a positive dimensional linear space of dimension at least
$$
\dim(E) \geq (m-n+1) b.
$$

Here, we have no guarantee on the dimension of the locus of critical points, but we know that at any critical point $p$ of $L$, the Hessian of $L$ at $p$ has at least $(m-n+1) b$ zero eigenvalues.  

Consider the data set fixed, so $a, b, n$ fixed.  Let us consider the growth of the dimensions of these various spaces as $m$ increases.  Note that the codimension of the star locus $S$ grows quadratically in $m$.  Same thing with the core locus $C$.  Meanwhile, the codimension of the locus of global minima $M$ is constant.  Finally, we have no lower bound on the dimension of the locus of critical points of $L$, even for large $m$ i.e. very wide networks.  While we can show that for $m$ large enough there are no isolated local minima or local maxima, we do not know any way to show that there are no isolated saddle points, and as far as we know that is possible.  

%state this as a prop, put in intro
We collect these calculations in the following Proposition.

\begin{proposition}\label{dimloci}
Consider a family of feedforward neural networks with hidden layers of increasing width training on a fixed data set.  That is, let $a, b,$ and $n$ be fixed while $m$ increases.  Then the dimensions of the locus of global minima, star locus, and core locus are:
\begin{align*}
\dim(M) &= (\ell-2)m^2 + (a+b+\ell-1)m + b(1-n),
\\
\dim(S) &=  (\ell-3)m^2 + (a+1)m,
\\
\dim(C) &=  (\ell-4)m^2 + (a+1)m,
\end{align*}
while the dimension of the locus of all critical points is unknown.
\end{proposition}

We can also consider the number of zero eigenvalues of $Hess(L)$ at various kinds of critical points and how the dimension of the zero eigenspace grows with $m$.  For any point $p$ in the core locus $C$, the number of nonzero eigenvalues of $Hess(L)$ at $p$ is always exactly 1.  Meanwhile, the number of nonzero eigenvalues of $Hess(L)$ at any point $p \in M$ also stays constant, namely is $n b$.  Finally, the number of \emph{zero} eigenvalues of $Hess(L)$ at any critical point $p$ of $L$ is at least $(m-n+1)b$, which grows linearly in $m$.  

\begin{proposition}\label{dimzeroeigen}
Consider a family of feedforward neural networks with hidden layers of increasing width training on a fixed data set.  That is, let $a, b,$ and $n$ be fixed while $m$ increases.  Then the number of zero eigenvalues of the Hessian of $L$ at any critical point in the core locus, locus of global minima, or any critical point are:
\begin{align*}
\begin{cases}
d-1 & \text{for } p \in C,
\\
d-bn & \text{for } p \in M,
\\
\text{at least } (m+1 - n) b & \text{for all other critical points  } p.
\end{cases} 
\end{align*}
\end{proposition}

We see that all three kinds of spaces grow in dimension as $m$ increases.  The locus of global minima grows most quickly, followed by the star locus, followed by the linear spaces identified in this paper within the level sets of critical points.  

Meanwhile, the number of zero eigenvalues of $Hess(L)$ also grows as $m$ increases, for all of these kinds of critical points.  For points in the core locus $C$ or locus $M$ of global minima, the number of nonzero eigenvalues of $Hess(L)$ is bounded above by a constant as $m$ increases.  On the other hand, for all other critical points, the number of zero eigenvalues of $Hess(L)$ is bounded below by a linear function in $m$.  
 
\subsection{Conclusion}
Much is mysterious about the geometry of the loss function $L$ in deep nonlinear neural networks.  Here, we have aimed to advance our understanding of the geometry of $L$ near its critical points.  When possible, we have given quantitative descriptions on the geometry of $L$ near $p$, such as bounds on the zero eigenvalues of the Hessian of $L$ at $p$, or bounds on the dimension of special loci containing $p$.

We describe three overparameterized regimes --- overparameterized network, very wide network, and extremely wide network.  In this paper, we analyze networks in the first two regimes, as well as the unconstrained case.  For networks of depth $\geq 3$ and any size, so both underparameterized and overparameterized networks, we identify a positive dimensional locus of critical points we call the star locus $S$, and within it, a positive dimensional locus of degenerate critical points we call the core $C$.  Furthermore, we compute the dimension of $S$ and $C$.  

For simply overparameterized networks, we recall from \cite{cooper} that the locus of global minima $M = L^{-1}(0)$ is a smooth manifold of dimension $d-bn$.  For very wide networks, it has recently been shown that $L$ has no spurious valleys  \cite{fullrank}, \cite{venturibruna}, \cite{ruoyumeasure0}.  Here, we extend those results by observing that $L$ cannot have isolated local minima or maxima (though we do not rule out the case of isolated saddle points).  Furthermore, in the case of any critical point $p$ which is not a global minimum, we give a lower bound on the dimension of a certain linear space containing $p$ and contained in the level set of $p$.  This applies to saddle points in addition to local minima or maxima, and implies a lower bound on the number of zero eigenvalues of the Hessian of $L$ at $p$.  However, it is stronger, in that it gives us some partial knowledge of the global geometry of $L$, not just the geometry of $L$ in a neighborhood of $p$.  

Together, these results provide some basic information about the geometry of $L$ near some of its critical points, in several regimes.  Our results on the star and core loci hold for all neural networks of depth $\geq 3$.  Our bounds on the dimension of the locus of global minima hold for overparameterized networks, whether simply overparameterized, very wide, or extremely wide.  Finally, our bounds on the dimension of the special linear spaces we describe containing all critical points of $L$ that are not global minima hold for very wide and extremely wide networks.  In the case of networks that are both depth $\geq 3$ and very wide, we can compare the relative dimensions of all these loci.  In that case we find three kinds of spaces grow in dimension as the width $m$ increases.  The locus of global minima grows most quickly, followed by the star locus, followed by the linear spaces identified in this paper within the level sets of critical points.

\bibliographystyle{is-alpha}
\bibliography{next_paper.bib}

\newcommand{\etalchar}[1]{$^{#1}$}
\begin{thebibliography}{KWL{\etalchar{+}}19}
\ifx \showCODEN  \undefined \def \showCODEN #1{CODEN #1}  \fi
\ifx \showISBN   \undefined \def \showISBN  #1{ISBN #1}   \fi
\ifx \showISSN   \undefined \def \showISSN  #1{ISSN #1}   \fi
\ifx \showLCCN   \undefined \def \showLCCN  #1{LCCN #1}   \fi
\ifx \showPRICE  \undefined \def \showPRICE #1{#1}        \fi
\ifx \showURL    \undefined \def \showURL {URL }          \fi
\ifx \path       \undefined \input path.sty               \fi
\ifx \ifshowURL \undefined
     \newif \ifshowURL
     \showURLtrue
\fi

\bibitem[ALS18]{zhu}
Zeyuan Allen{-}Zhu, Yuanzhi Li, and Zhao Song.
\newblock A convergence theory for deep learning via over-parameterization.
\newblock {\em CoRR}, abs/1811.03962, 2018.
\newblock \ifshowURL {\showURL \path|http://arxiv.org/abs/1811.03962|}\fi.

\bibitem[Ber16]{nonlinearprogramming}
Dimitri~P. Bertsekas.
\newblock {\em Nonlinear programming}.
\newblock Athena Scientific Optimization and Computation Series. Athena
  Scientific, Belmont, MA, third edition, 2016.
\newblock \showISBN{978-1-886529-05-2; 1-886529-05-1}.
\newblock xviii+861 pp.

\bibitem[Coo18]{cooper}
Yaim Cooper.
\newblock The loss landscape of overparameterized neural networks.
\newblock {\em CoRR}, abs/1804.10200, 2018.
\newblock \ifshowURL {\showURL \path|http://arxiv.org/abs/1804.10200|}\fi.

\bibitem[DLL{\etalchar{+}}19]{dulee}
Simon Du, Jason Lee, Haochuan Li, Liwei Wang, and Xiyu Zhai.
\newblock Gradient descent finds global minima of deep neural networks.
\newblock In Kamalika Chaudhuri and Ruslan Salakhutdinov, editors, {\em
  Proceedings of the 36th International Conference on Machine Learning},
  volume~97 of {\em Proceedings of Machine Learning Research}, pages
  1675--1685. PMLR, Long Beach, California, USA, 09--15 Jun 2019.
\newblock \ifshowURL {\showURL
  \path|http://proceedings.mlr.press/v97/du19c.html|}\fi.

\bibitem[DLS19]{ruoyumin}
Tian Ding, Dawei Li, and Ruoyu Sun.
\newblock Sub-optimal local minima exist for almost all over-parameterized
  neural networks.
\newblock {\em ArXiv}, abs/1911.01413, 2019.

\bibitem[Dun87]{dun}
J.~C. Dunn.
\newblock On the convergence of projected gradient processes to singular
  critical points.
\newblock {\em J. Optim. Theory Appl.}, 55\penalty0 (2):\penalty0 203--216,
  1987.
\newblock \showISSN{0022-3239}.
\newblock \ifshowURL {\showURL \path|https://doi.org/10.1007/BF00939081|}\fi.

\bibitem[DVSH18]{NoBarriers}
Felix Draxler, Kambis Veschgini, Manfred Salmhofer, and Fred Hamprecht.
\newblock Essentially no barriers in neural network energy landscape.
\newblock In {\em Proceedings of the 35th International Conference on Machine
  Learning}, volume~80, pages 1309--1318, 2018.

\bibitem[GIP{\etalchar{+}}18]{LossSurfaces}
Timur Garipov, Pavel Izmailov, Dmitrii Podoprikhin, Dmitry~P Vetrov, and
  Andrew~G Wilson.
\newblock Loss surfaces, mode connectivity, and fast ensembling of dnns.
\newblock In {\em Advances in Neural Information Processing Systems}, pages
  8789--8798, 2018.

\bibitem[JGH18]{ntk}
Arthur Jacot, Franck Gabriel, and Clement Hongler.
\newblock Neural tangent kernel: Convergence and generalization in neural
  networks.
\newblock In S.~Bengio, H.~Wallach, H.~Larochelle, K.~Grauman, N.~Cesa-Bianchi,
  and R.~Garnett, editors, {\em Advances in Neural Information Processing
  Systems 31}, pages 8571--8580. Curran Associates, Inc., 2018.
\newblock \ifshowURL {\showURL
  \path|http://papers.nips.cc/paper/8076-neural-tangent-kernel-convergence-and-generalization-in-neural-networks.pdf|}\fi.

\bibitem[JGN{\etalchar{+}}17]{jordansaddle}
Chi Jin, Rong Ge, Praneeth Netrapalli, Sham~M. Kakade, and Michael~I. Jordan.
\newblock How to escape saddle points efficiently.
\newblock {\em CoRR}, abs/1703.00887, 2017.
\newblock \ifshowURL {\showURL \path|http://arxiv.org/abs/1703.00887|}\fi.

\bibitem[Kaw16]{NoPoor}
Kenji Kawaguchi.
\newblock Deep learning without poor local minima.
\newblock In {\em Advances in neural information processing systems}, pages
  586--594, 2016.

\bibitem[KWL{\etalchar{+}}19]{sanjeevepsilonconnected}
Rohith Kuditipudi, Xiang Wang, Holden Lee, Yi~Zhang, Zhiyuan Li, Wei Hu,
  Sanjeev Arora, and Rong Ge.
\newblock Explaining landscape connectivity of low-cost solutions for
  multilayer nets.
\newblock {\em CoRR}, abs/1906.06247, 2019.
\newblock \ifshowURL {\showURL \path|http://arxiv.org/abs/1906.06247|}\fi.

\bibitem[LDS18]{ruoyumeasure0}
Dawei Li, Tian Ding, and Ruoyu Sun.
\newblock Over-parameterized deep neural networks have no strict local minima
  for any continuous activations.
\newblock {\em CoRR}, abs/1812.11039, 2018.
\newblock \ifshowURL {\showURL \path|http://arxiv.org/abs/1812.11039|}\fi.

\bibitem[Lee09]{lee}
Jeffrey~M. Lee.
\newblock {\em Manifolds and differential geometry}, volume 107 of {\em
  Graduate Studies in Mathematics}.
\newblock American Mathematical Society, Providence, RI, 2009.
\newblock \showISBN{978-0-8218-4815-9}.
\newblock xiv+671 pp.
\newblock \ifshowURL {\showURL \path|https://doi.org/10.1090/gsm/107|}\fi.

\bibitem[Ngu19]{quynhconnectedsublevel}
Quynh Nguyen.
\newblock On connected sublevel sets in deep learning.
\newblock {\em CoRR}, abs/1901.07417, 2019.
\newblock \ifshowURL {\showURL \path|http://arxiv.org/abs/1901.07417|}\fi.

\bibitem[NMH18]{fullrank}
Quynh Nguyen, Mahesh~Chandra Mukkamala, and Matthias Hein.
\newblock On the loss landscape of a class of deep neural networks with no bad
  local valleys.
\newblock {\em CoRR}, abs/1809.10749, 2018.
\newblock \ifshowURL {\showURL \path|http://arxiv.org/abs/1809.10749|}\fi.

\bibitem[NW06]{numericaloptimization}
Jorge Nocedal and Stephen~J. Wright.
\newblock {\em Numerical optimization}.
\newblock Springer Series in Operations Research and Financial Engineering.
  Springer, New York, second edition, 2006.
\newblock \showISBN{978-0387-30303-1; 0-387-30303-0}.
\newblock xxii+664 pp.

\bibitem[SS17]{spuriouscommon}
Itay Safran and Ohad Shamir.
\newblock Spurious local minima are common in two-layer relu neural networks.
\newblock {\em arXiv:1712.08968}, 2017.

\bibitem[VBB18]{venturibruna}
Luca Venturi, Afonso~S. Bandeira, and Joan Bruna.
\newblock Spurious valleys in two-layer neural network optimization landscapes.
\newblock 2018.
\newblock \ifshowURL {\showURL \path|https://arxiv.org/pdf/1802.06384v2|}\fi.

\bibitem[YSJ18]{srabadlocalmin}
Chulhee Yun, Suvrit Sra, and Ali Jadbabaie.
\newblock A critical view of global optimality in deep learning.
\newblock {\em CoRR}, abs/1802.03487, 2018.
\newblock \ifshowURL {\showURL \path|http://arxiv.org/abs/1802.03487|}\fi.

\end{thebibliography}

\end{document}